\documentclass{article}

\usepackage{PRIMEarxiv}
\usepackage[utf8]{inputenc} 
\usepackage[T1]{fontenc}    
\usepackage{hyperref}       
\usepackage{url}            
\usepackage{booktabs}       
\usepackage{amsfonts}       
\usepackage{nicefrac}       
\usepackage{microtype}      
\usepackage{lipsum}
\usepackage{fancyhdr}       
\usepackage{graphicx}       
\usepackage[table,xcdraw]{xcolor}
\usepackage{enumitem}
\setlist{noitemsep}
\setlist{nolistsep}
\usepackage{amsmath}
\usepackage{amsthm}
\usepackage{multirow}
\usepackage{cleveref}
\graphicspath{{media/}}     
\usepackage[linesnumbered,ruled,vlined]{algorithm2e}
\title{Mitigating Data Heterogeneity in Federated Learning with Data Augmentation}

\author{Artur Back de Luca\thanks{Equal contribution.}\hspace{0.4em}$^{1}$, Guojun Zhang$^{*1}$, Xi Chen$^{1}$ and Yaoliang Yu$^{2}$\\
\texttt{\{arturb.luca, guojun.zhang,  xi.chen4\}@huawei.com} \\
\texttt{yaoliang.yu@uwaterloo.ca} \\
 $^1$Huawei Noah's Ark Lab \\ 
 $^2$University of Waterloo
}

\newcommand{\x}[1]{\theta^{(#1)}}
\newcommand{\Sc}{\mathcal{S}}
\newcommand{\Dc}{\mathcal{D}}
\newcommand{\Xc}{\mathcal{X}}

\newcommand{\Eb}{\mathbb{E}}
\newcommand{\Rb}{\mathbb{R}}

\def \supp {\mathop{\tt supp}}
\def \one {\mathbf{1}}
\def \zero {\mathbf{0}}
\def \proj {\mathop{\rm proj}}
\def \argmin {\mathop{\rm argmin}}

\newtheorem{thm}{Theorem}

\newtheorem{prop}[thm]{Proposition}

\begin{document}
\maketitle

\begin{abstract}
    Federated Learning (FL) is a prominent framework that enables training a centralized model while securing user privacy by fusing local, decentralized models.
    In this setting, one major obstacle is data heterogeneity, i.e., each client having non-identically and independently distributed (non-IID) data.
    This is analogous to the context of Domain Generalization (DG), where each client can be treated as a different domain.
    However, while many approaches in DG tackle data heterogeneity from the algorithmic perspective, recent evidence suggests that data augmentation can induce equal or greater performance.
    Motivated by this connection, we present federated versions of popular DG algorithms, and show that by applying appropriate data augmentation, we can mitigate data heterogeneity in the federated setting, and obtain higher accuracy on unseen clients.
    Equipped with data augmentation, we can achieve state-of-the-art performance using even the most basic Federated Averaging algorithm, with much sparser communication.
\end{abstract}

\section{Introduction}
With the increasing computation power of edge devices, Federated Learning (FL) \cite{mcmahan2017communication} provides a solution to data privacy constraints using distributed computation.
In this paradigm, each user (or client) has a dedicated model that is trained locally and solely on the said user's data collection.
During this process, a server is responsible for collecting these local models. By combining their parameters in a principled way, Federated Learning aims to generate a powerful centralized model.

One important aspect of this distributed setting is that each client's dataset has a unique manifestation,
which generally reflects a fraction of data distribution of all clients.
This implies that the samples of each client are not drawn from the combined distribution in an independently and identically distributed (IID) fashion (i.e.~non-IID).
For instance, a user's image collection may have certain subjects of interest, backgrounds, lighting conditions and viewpoints.
With this considerable limitation, local models will vary substantially, and when aggregated, will render poor generalization \cite{zhao2018federated}.

The compartmental view of data share similarities with the popular field of Domain Generalization (DG) \cite{blanchard2011generalizing}. In this framework, instead of having observations drawn from a single distribution, we assume that data is originated from multiple domains. With this setting, the goal of DG is to find a robust classifier that generalizes beyond the domains seen during training. Different strategies in Domain Generalization have successfully improved out-of-distribution performance by modifying the objective function accounting for the different domains \cite{arjovsky2019invariant, sagawa2019distributionally, krueger2021out}. Drawing a connection to the federated setting, each domain may correspond to one or more clients, and thus many DG algorithms can be transferred to the FL setting (See \S~\ref{sec:bridge}), though not always easily.

However, recent evidence \cite{gulrajani2020search} suggests that, with all conditions equal, when equipped with data augmentation, no algorithm outperforms the elementary Empirical Risk Minimization (ERM) algorithm. 
An intuitive explanation is that data augmentation acts as an invariance mechanism, discarding potential spurious correlations that are commonly encountered in each domain. 

Motivated by the aforementioned results, and considering the data heterogeneity present in the federated setting, in this work, we show that data augmentation also poses a straightforward yet effective solution to mitigate these differences in FL.
By selecting transformations that are pertinent to the overall distribution, we can apply data augmentation to each client data, generating more well-distributed and similar data distributions, while maintaining privacy.

\textbf{Contributions.} We study \emph{Federated Domain Generalization} as shown in Figure \ref{fig:federated_diagram}. Each client owns a proprietary dataset, and the goal is to learn a model through federated learning, such that it generalizes both in-domain (ID) and out-of-domain (OOD).\footnote{The term \textit{out-of-domain} is here used instead of the more prominent and equivalent term \textit{out-of-distribution} to keep consistency with the presented formulation and avoid potential ambiguities}
The in-domain datasets are the private ones of the participating clients during training, and a out-of-domain dataset resides in a client that has not participated in the federated training process.
This setting bridges the fields of Domain Generalization and Federated Learning.
Motivated by the success of data augmentation in DG, we show that, from a \emph{causal perspective}, data augmentation can alleviate the non-IIDness in FL, improving both generalization and convergence:

\begin{itemize}
\item \textbf{Generalization:} with \emph{proper} data augmentation, Federated Learning gives better OOD performance, and even the basic FedAvg algorithm compares favorably against state-of-the-art alternatives.
\item \textbf{Convergence:} Since the non-IIDness is mitigated with data augmentation, one can perform much sparser communication to achieve comparable performance as the centralized setting, with as few as two communication rounds.
\end{itemize}

\begin{figure}[ht]
    \centering
    \includegraphics[width=0.8\textwidth]{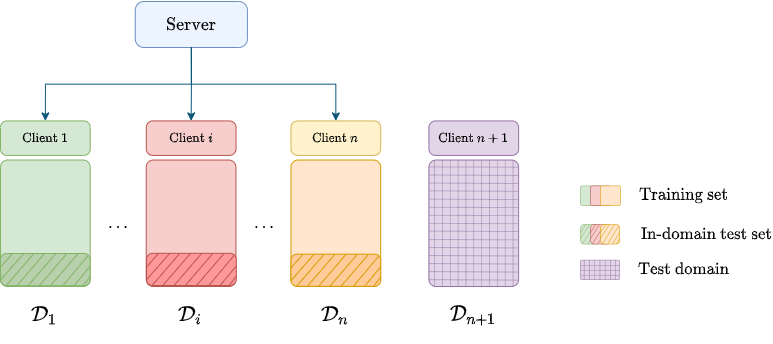}
    \caption{The setting of Federated Domain Generalization. By training on client training sets in a federated manner, we hope to learn a model that generalizes well both in-domain (i.e.~on in-domain test sets) and out-of-domain (i.e.~on the unseen test domain).}
    \label{fig:federated_diagram}
\end{figure}

\vspace{-0.3em}
\section{The non-IIDness of Federated Learning}\label{sec:noniid_fl}
\vspace{-0.3em}


One of the main challenges of FL is \emph{data heterogeneity}, i.e. the dissimilarity between each client's data distribution. Suppose $\Dc_i$ is the data distribution for client $i$, then the non-IIDness means that $\Dc_i$ and $\Dc_j$ could differ significantly for $i\neq j$, while in conventional IID distributed learning, $\Dc_i = \Dc_j$ for any $i\neq j$. The goal of federated learning is to find a global minimizer $\theta^*$ of the total loss,
\begin{align}
    \label{eq:fedlearn}
    \mbox{such that }\theta^* \in \underset{\theta}{\rm argmin}\, 
    F(\theta) := \frac{1}{n}\sum_i^n f_i(\theta),
    \quad \textrm{ with }
    f_i(\theta) := \Eb_{(x, y)\sim\mathcal{D}_i}\left[\ell(\hat{y}_\theta(x), y)\right].
\end{align}
Here $(x, y) \sim \Dc_i$ is an (input, label) pair drawn from $\Dc_i$, and $\ell$ is the loss function (e.g.~cross entropy) that measures the discrepancy between the model prediction $\hat{y}_\theta(x)$ and the true label $y$. The model $\theta$, usually parametrized as deep neural networks, takes an input $x$ and output a probability vector for prediction. In practice, we draw samples from each distribution $\Dc_i$ to estimate the expectation. 


In the IID case, the gradient $\nabla f_i$ is also an unbiased estimate of $\nabla F$, implying that the minimizer $\theta^*$ of (\ref{eq:fedlearn}) is also (approximately) a minimizer of the individual $f_i$'s.
To save communication cost, we can compute the \emph{one-shot average of local models} \cite{zinkevich2010parallelized, zhang2012communication}:
\begin{equation}
    \label{eq:oneshot}
    \overline{\theta} = \frac{1}{n}\sum_i^n \theta_i^*, \textrm{ with }\theta_i^* \in \underset{\theta}{\rm argmin}\, f_i(\theta),
\end{equation}
as a surrogate for $\theta^*$. The optimality gap $F(\overline{\theta}) - F({\theta}^*)$ can be controlled and decreases with the number of samples for each client \cite{zhang2012communication}.

However, under non-IIDness, the optimality gap can be much larger. To overcome this difficulty, \textit{Federated Averaging} (\textit{FedAvg}) \cite{mcmahan2017communication} proposes to trade the communication cost for better performance. Instead of performing local updates until convergence with one communication round, in FedAvg, we take more communication rounds with fewer local steps. 
At communication round $t$, the centralized model $\theta_t$ is dispatched to the clients, where it undergoes training over the local data. After $E$ local steps, each client transmits to the server its updated model $\theta_{t+1}^{(i)}$, which are then averaged:
\begin{align}
\textstyle 
\theta_{t+1} = \frac{1}{n}\sum_{i=1}^n \theta_{t+1}^{(i)}.
\end{align}



For the non-IID case, the pattern extraction performed by each local model is going to be distinct, and their aggregation is likely to produce sub-optimal performance \cite{pathak2020fedsplit}. As mentioned in Theorem 5 of \cite{khaled2020tighter}, under mild assumptions, the optimality gap can be bounded as:
\begin{align}
\textstyle F\left(\frac{1}{T}\sum_{t=1}^T \theta_t\right) - F(\theta^*) \leq O\left(\frac{1}{T}\right) +  {\tt const.}\times \sigma^2(\theta^*)
\end{align}
where the constant depends on the learning rate $\eta$, the number of local steps and etc. The non-IIDness, or the data heterogeneity is measured by the following quantity:
\begin{align}\label{eq:heterogeneity}
\textstyle \sigma^2(\theta^*) = \frac{1}{n}\sum_{i=1}^n \|\nabla f_i(\theta^*)\|^2.
\end{align}

Since FedAvg, there has been great effort devoted to devising better optimization algorithms for addressing the non-IIDness in FL. For instance, FedProx \cite{li2020federated} proposes to use the proximal operator instead of several steps of gradient descent; Agnostic FL \cite{mohri2019agnostic} proposes to minimize the worst-case loss function among clients. More related work about this direction can be found in \S~\ref{sec:related_work}.

\section{Bridging Domain Generalization and Federated Learning}\label{sec:bridge}

In this section, we introduce Domain Generalization (DG), and show that, despite seemingly disconnected, DG and FL share the same principles.

The problem of domain generalization \cite{blanchard2011generalizing,muandet2013domain} assumes that we are given $n$ source domains $\Dc_1, \dots, \Dc_n$ with labeled data. By training a model that behaves well on all source domains, one hopes to generalize to a new unseen target domain $\Dc_{n+1}$. The prominent algorithm for solving this problem is Empirical Risk Minimization (ERM) \cite{vapnik1992principles}:
\begin{align}
\label{eq:ERM}
\min_{\theta} \sum_{i=1}^n f_i(\theta), \textrm{ with }f_i(\theta) := \Eb_{(x, y)\sim \Dc_i} [\ell(\hat{y}_\theta(x), y)].
\end{align}



If we compare \eqref{eq:ERM} and \eqref{eq:fedlearn}, we can easily draw the conclusion that \emph{ERM and FedAvg essentially follow the same principle,} if we treat each client distribution in FL as a source domain in DG. The critical difference is that in FL, the client datasets are kept locally due to privacy restriction, but in DG these datasets can be shared. Therefore, \emph{FL algorithms can be treated as a special category of DG algorithms}. From this bridge, one can connect DG algorithms with FL algorithms, and bring insights to FL from the active field of domain generalization. 

Table \ref{tab:dg_fl} summarizes the connection between DG and FL. As another example, GroupDRO \cite{sagawa2019distributionally} and AFL \cite{mohri2019agnostic} follow the same principle of minimizing the worst-case client (or source domain):
\begin{align}\label{eq:AFL}
\min_{\theta} \max_{\lambda \in \Delta} \sum_{i=1}^n \lambda_i f_i(\theta),
\end{align}
where $\Delta := \{\lambda \in \Rb^n: \one^\top \lambda = 1, \lambda \geq \zero\}$ is the probability simplex. The probability simplex can be generalized to $\Delta(\lambda_{\min}) :=\{\lambda \in \Rb^n: \one^\top \lambda = 1, \, \lambda \geq \lambda_{\min}\one \}$, from which the DG algorithm called MM-REx \cite{krueger2021out} has been designed. We can similarly write down the federated version, that we call \emph{Generalized-AFL} (Gen-AFL).
Moreover, we use this connection to devise another algorithm, named \emph{Variance Minimization} (VM), from \cite{krueger2021out}. Essentially, by replacing the gradient $\nabla f_i$ with the \emph{pseudo-gradient}, i.e., the sum of local gradient steps for client $i$, one can obtain the corresponding FL algorithm from a DG algorithm. However, not every DG algorithm can be easily transferable. For example, the popular IRM \cite{arjovsky2019invariant} algorithm suffers from gradient sensitivity after we move it to the FL setting. Full details of Gen-AFL, VM, and IRM can be found in Appendix~\ref{sec:fl_dg_alg}.


\begin{table*}[tb]
    \caption{Federated learning (FL) algorithms vs.~Domain Generalization (DG) algorithms. {\bf ERM:} Empirical Risk Minimization; {\bf GroupDRO:} Group distributional robust optimization; {\bf MM-REx:} Minimax Risk Extrapolation; {\bf V-REx:} Variance Risk Extrapolation; {\bf FedAvg:} Federated Averaging; {\bf AFL:} Agnostic Federated Learning; {\bf Gen-AFL:} Generalized AFL; {\bf VM:} Variance minimization.}
    \centering
\begin{tabular}{ccccc}
\toprule 
\textbf{Domain}         & ERM                                              & GroupDRO                                           & MM-REx                                 & V-REx                                  \\
\textbf{Generalization} & \cite{vapnik1992principles}     & \cite{sagawa2019distributionally} & \cite{krueger2021out} & \cite{krueger2021out} \\ \midrule
\textbf{Federated}      & FedAvg                                           & AFL                                                & \multirow{2}{*}{Gen-AFL}                                & \multirow{2}{*}{VM}                                     \\
\textbf{Learning}       & \cite{mcmahan2017communication} & \cite{mohri2019agnostic}          &                               &                               \\ \bottomrule
\end{tabular}%
    \label{tab:dg_fl}
\end{table*}



\subsection{Data augmentation in domain generalization}

The similarity between domain generalization and federated learning motivates us to bring insights from one field to the other. As explained in \cite{gulrajani2020search}, \emph{with proper data augmentation, ERM has the state-of-the-art performance, compared to a large variety of DG algorithms}. Does the same conclusion hold for federated learning, with the private source domains and sparse communication? To answer this question, we study the effect of data augmentation in \S~\ref{sec:DA4FL}.

Due to the similarity between DG and FL, we will use the word ``client'' and ``environment'' interchangeably when not causing confusion. 


\vspace{-0.3em}
\section{Data Augmentation for Federated Learning}\label{sec:DA4FL}
\vspace{-0.3em}

Inspired by the success of data augmentation in domain generalization, in this section, we propose to mitigate the non-IIDness in federated learning from the data perspective. This is orthogonal to the algorithmic approach to federated learning, e.g.~\cite{mohri2019agnostic, li2020federated}.

Our key observation is that data augmentation can reduce the heterogeneity of client distributions. As discussed in \S~\ref{sec:noniid_fl}, federated learning is quite challenging in the presence of non-IID data. By applying proper data augmentation to each local client, the client distributions are more similar and 
FL will suffer less from the non-IID issue. 

\textbf{Data augmentation} is a widely used strategy in machine learning for increasing the diversity of samples. It often allows us to obtain more robust models with better performance effectively. In this work, we define data augmentation as transformations in the input space. More specifically, denote $\Xc$ as the input space, a \emph{data augmentation} function $T$ is defined as $T: \mathcal{X}\rightarrow\mathcal{X}$. In supervised learning, data augmentation is typically constructed to preserve ground-truth labels.


\vspace{-0.3em}
\subsection{Data augmentation with different environments: a causal approach}\label{sec:causal}
\vspace{-0.3em}

To better understand how data augmentation can improve similarity among different clients (environments), we propose a \emph{causal model} that explains this approach, as illustrated in Figure \ref{fig:causal_diagram_rmnist}. 
The key observation we make in the causal model is that all environments are generated from a common cause $Z$, which is a random variable. This random variable is sometimes called the \emph{semantic content} \cite{gatys2016image}, which is generated from the true label $Y$. For example, we could generate the digit images ($Z$) from digit labels ($Y$). This data generating principle is also known as anti-causal \cite{scholkopf2012causal}.

However, individuals write digits differently, e.g.~handwriting can vary. Besides, even the same person can write the same digit in different ways. Therefore, for a client $i$, we use a random variable $\varepsilon_i$ to account for these variations.
The causal mechanism allows us to generate samples for each client, represented by random variable $X_i$:
\begin{align}
X_i := g(Z, \varepsilon_i), \, \textrm{ with }Z\perp \varepsilon_i.
\end{align}
This is known as the \textbf{Structural causal Model} (SCM) \cite{peters2017elements}. Note that $Z$ and $\varepsilon_i$ are independent, and the causal mechanism $g$ remains invariant across environments. For example, in Rotated MNIST \cite{ghifary2015domain} $Z$ represents the original digit images, and each $\varepsilon_i$ represents a fixed angle. The function $g$ is the rotation operation of $Z$ with the angle information given by $\varepsilon_i$. 

So far, we have talked about the data generation process from the causal perspective.
Now we discuss the importance of data augmentation in federated domain generalization, which is two-fold.
First, we use data augmentation to \emph{homogenize} different domains that contain the same semantic meaning but various environmental factors (e.g.~the rotation angle in Rotated MNIST, the style of painting the same object). Essentially, if we have a general sense of the underlying SCM, then we can take advantage of it to generate new samples, by applying proper data augmentation. In the ideal case, the data augmentation $T$ for each $X_i := g(Z, \varepsilon_i)$ induces a transformation $T'$ for $\varepsilon_i$, without modifying $Z$. Eventually, we want all the transformed environmental variables $T'(\varepsilon_i)$ to match an environment-agnostic variable $\overline{\varepsilon}$.
With clients distributions more aligned, we have less data heterogeneity and therefore, better \emph{convergence}.

Augmentation also has the role of promoting \emph{generalization}.
The data augmentation $T$ can expand the training distributions, potentially incorporating unseen domains, thus entailing better OOD generalization.

\begin{figure}[ht]
    \centering
    \includegraphics[width=0.9\textwidth]{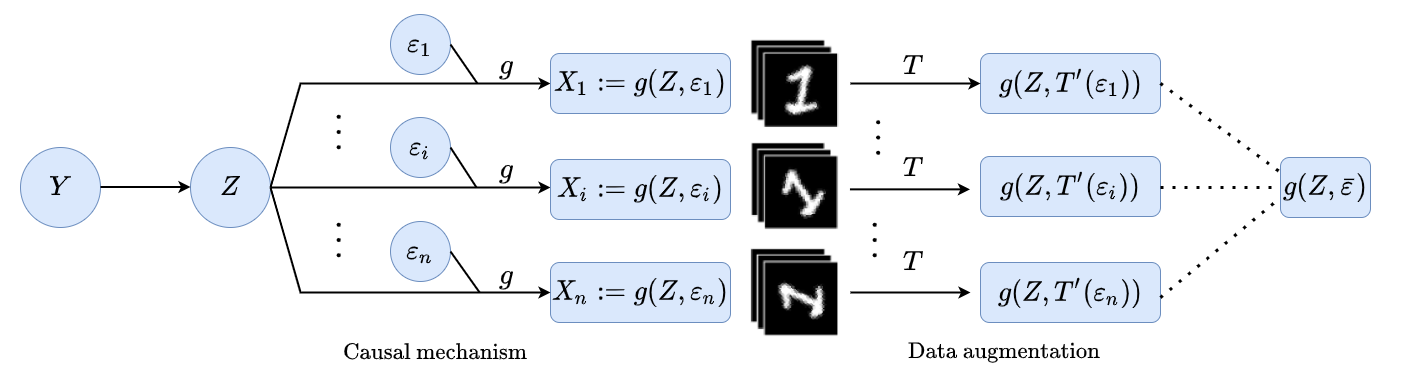}
    \caption{A causal model for data augmentation. $Y$: the random variable for the label; $Z$: the random variable for the semantic meaning; $X_i$: the random variable for the samples in environment $i$; $\varepsilon_i$: the random variable for the environmental factor in environment $i$; $X_i = g(Z, \varepsilon_i)$ denotes the structural causal equation for the data generating process; See \S~\ref{sec:rmnist} for a concrete example.}
    \vspace{-0.8em}
    \label{fig:causal_diagram_rmnist}
\end{figure}

Choosing the ``right'' data augmentation to the ``right'' extent is vital for our approach, after all, data augmentation needs \emph{prior knowledge} of the data generating process $g(Z, \varepsilon)$, which may not be possible.
In such cases, data transformations may not be capable of mapping all environmental factors to an agnostic representation.
Nevertheless, even with augmentations marginally related to the causal structure, we can still expand training distributions, increasing their overlap and favoring generalization.
We further demonstrate this in the following case study, as well as in the experimental results of \S~\ref{sec:exp}.

\subsection{Rotated MNIST: a case study}\label{sec:rmnist}

\textbf{Formulation.} We give a concrete example for the causal model with the Rotated MNIST dataset \cite{ghifary2015domain}. In this example, we have:
\begin{align}
&Y: \textrm{the label of the digits;  }Z: \textrm{the original MNIST digit images; } \\
&\varepsilon_i: \textrm{the rotation angle for the digits;  }X_i: \textrm{the rotated digit images}.
\end{align}
Specifically, each $\varepsilon_i$ follows a Dirac distribution $\delta_{\varepsilon = t_i}$, with $\{t_1, t_2, \dots, t_5\} = \{0^\circ, 15^\circ, \dots, 60^\circ\}$. The data generating function $g$ now is:
\begin{align}
g(Z, \varepsilon_i): \textrm{ rotate $Z$ with angle $\varepsilon_i$}.
\end{align}
The data augmentation we consider is:
\begin{align}\label{eq:DA_RMNIST}
T_{\alpha}(X_i): \textrm{ randomly \emph{rotate} $X_i$ with angle $\beta$, which is uniformly drawn from }[-\alpha, \alpha].
\end{align}
Let us now study how the data augmentation encourages different client distributions to be similar. We can see that after the data augmentation, we have:
\begin{align}
T_\alpha(g(Z, \varepsilon_i)) = g(Z, T'_\alpha(\varepsilon_i)), \textrm{ where } T'_{\alpha}(\varepsilon_i): \textrm{ randomly \emph{add} $\varepsilon_i$ with angle }\beta\sim U[-\alpha, \alpha].
\end{align}

This equation tells us that applying the data augmentation $T_\alpha$ onto $g(Z, \varepsilon)$ is equivalent to applying a related transformation $T'_\alpha$ directly on the environmental variable $\varepsilon$. In general, based on our (perhaps incomplete) prior knowledge of the data generation $g$, we can devise appropriate data augmentation $T$ that solely induces operation on $\varepsilon_i$. However, if the data augmentation $T$ is not related to the data generation process $g$, such as Gaussian blur, then applying $T$ onto $g(Z, \varepsilon)$ may reduce data heterogeneity,
but will also inevitably modify the semantic variable $Z$, and degrade the generalization (see also \S~\ref{sec:exp}). 




\textbf{Proper data augmentation can homogenize client domains.} We now show that the domains for clients are indeed more similar after proper data augmentation.
Since the data generation function $g$ remains the same throughout, it suffices to measure the distance between the joint distributions of $(Z, \varepsilon_i)$. We use total variation as the measure. Suppose we have two distributions with density functions $p_1$ and $p_2$, and the union of their supports to be $\supp$. Then the total variation (TV) distance between the two distributions is:
\begin{align}\label{eq:def_TV}
\textstyle d_{\rm TV}(p_1, p_2) = \int_{u\in \supp}|p_1(u) - p_2(u)| du.
\end{align}


From the independence assumption of our SCM model, the joint distribution of $(Z, \varepsilon)$ is the product of marginals. Hence, for $(Z, \varepsilon_1) \sim p_Z \times p_1$ and $(Z, \varepsilon_2) \sim p_Z \times p_2$, the TV distance is:
\begin{align}
d_{\rm TV}( p_Z\times p_1, p_Z \times p_2) 
= d_{\rm TV}(p_1, p_2).
\end{align}
Therefore, it suffices to study the TV distance for the distributions of the environmental variable $\varepsilon$. 
Before the data augmentation, the TV distance is between two Dirac distributions $\delta_{\varepsilon = t_1}$ and $\delta_{\varepsilon = t_2}$:
\begin{align}
d_{\rm TV}(\delta_{\varepsilon = t_1}, \delta_{\varepsilon = t_2}) = 2.
\end{align}
We use $U(a, b)$ to represent a uniform distribution on the interval $[a, b]$. After the data augmentation $T_\alpha$, the TV distance is between two uniform distributions, $U(t_1 - \alpha, t_1 + \alpha)$ and $U(t_2 - \alpha, t_2 + \alpha)$. WLOG, we assume that $t_1 < t_2$. With the definition of uniform distribution, we find that:
\begin{align}
d_{\rm TV}(U(t_1 - \alpha, t_1 + \alpha), U(t_2 - \alpha, t_2 + \alpha)) = \frac{t_2 - t_1}{\alpha} \textrm{ if }\alpha > \frac{t_2 - t_1}{2},
\end{align}
and $d_{\rm TV}(U(t_1 - \alpha, t_1 + \alpha), U(t_2 - \alpha, t_2 + \alpha)) = 2$ otherwise. This means that with \emph{increasing data augmentation} (larger $\alpha$), different domains would look more similar. 



\section{Experiments}\label{sec:exp}

In this section, we examine the generalization performance of FL algorithms and show how data augmentation addresses the non-IIDness in federated learning, and thus the communication issue.
\subsection{Experimental setup}\label{sec:experiment_setup}

Since our task is federated domain generalization, we borrow datasets from domain generalization to set up our federated learning experiments. We assume full participation throughout.  

\textbf{Rotated MNIST \cite{ghifary2015domain}} is a domain generalization dataset adapted from MNIST \cite{lecun1998mnist}. In this dataset, each domain 
has its digits rotated by a fixed angle, from $\{0^\circ, 15^\circ, 30^\circ, 45^\circ, 60^\circ, 75^\circ\}$. 
In-distribution performance is evaluated over a validation set, 10\% of the training domain, while the domain of 75$^\circ$ is set aside for OOD.
For this task, we use the convolutional network as reported in \cite{zhang2022equality}, trained with batch size of 64 and learning rate of {\tt 1e-3} with the Adam optimizer \cite{kingma2015adam}. For each client, we take the local steps $E = 200$ \cite{mcmahan2017communication}.

\textbf{PACS \cite{li2017deeper}}
is an image dataset categorized into $10$ classes that are scattered across four different domains, each having a distinct trait: photograph (P), art (A), cartoon (C) and sketch (S).
For this experiment, the sketch domain (S) is set aside for testing.

For the federated experiments in PACS, the training domains are further divided into four subsets, with 12 clients in total. The splitting method is based on \cite{hsu2019measuring}, using a symmetric Dirichlet distribution (with parameter $200$) to assign examples from each class in non-IID fashion.
For this task we use ResNet-18 \cite{he2016deep}, pretrained on ImageNet, with batch size of 32 and learning rate of {\tt 5e-5}. We use the Adam optimizer \cite{kingma2015adam} and for each client, we take the local step $E = 200$ \cite{mcmahan2017communication}.

\textbf{OfficeHome \cite{venkateswara2017deep}} is analogous to PACS, having four distinct image domains: Art (A), ClipArt (C), Product (P) and Real-World (R). However, in OfficeHome the number of classes is much greater, with 65 categories in total. In this experiment, we use the same settings as reported in PACS, using the ClipArt domain for testing. For further discussion on the datasets and experiments, see App.~\ref{sec:benchmark}.

\begin{figure}[h]
    \centering
    \includegraphics[width=\textwidth]{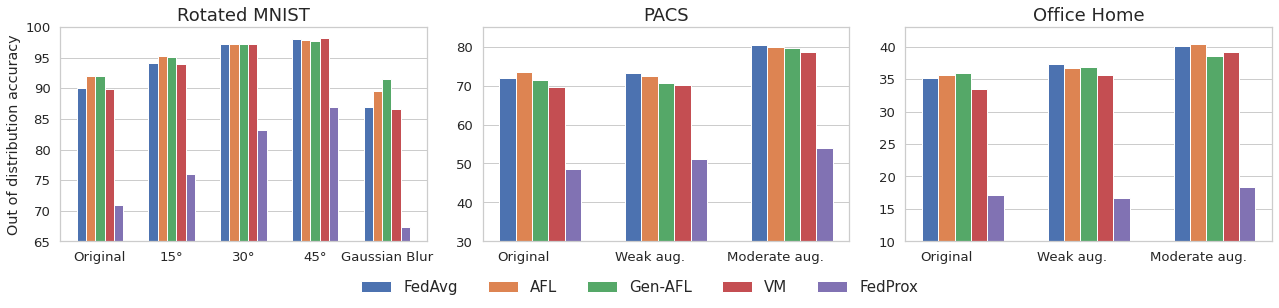}
    \caption{Out of distribution accuracy of federated algorithms across different data augmentation schemes. Augmentations are further described in \S~\ref{sec:augmentations}.}
    \label{fig:results_experiment}
    \vspace{-0.5em}
\end{figure}

\subsection{The importance of an appropriate augmentation}\label{sec:augmentations}

\paragraph{Augmentation methods used.} As reported in \S~\ref{sec:rmnist}, akin to the causal mechanism that generates different domains, in Rotated MNIST, we apply increasing levels of random rotation in \eqref{eq:DA_RMNIST}, with $\alpha \in \{0^\circ, 15^\circ, 30^\circ, 45^\circ, 60^\circ\}$. We also measure the effect of applying uncorrelated transformations, such as Gaussian blur. Since the causal mechanism is not as patent in PACS and OfficeHome, we measure the effect of stacking multiple transformations that are pertinent in distinguishing the different domains and do not compromise the semantic content.
We first apply positional transformations, such as random cropping and horizontal flipping, here denoted as \emph{weak augmentation}.
In \emph{moderated augmentation}, we further include color transformations, such as random conversion to grayscale and color jittering. 

In general, out-of-domain accuracies increase with stronger levels of suitable data augmentation (see Figure \ref{fig:results_experiment}). 
This verifies that suitable data augmentation can significantly improve out-of-domain generalization, even in the federated learning setting. As we explained in \S~\ref{sec:rmnist}, this is because 
our data augmentation alleviates the data heterogeneity of training and test clients.

However, for transformations that are irrelevant to the task, such as Gaussian blur in Rotated MNIST, the performance is even worse than no data augmentation.
This shows us the importance of choosing the right data augmentation that is relevant to the data generating process, as otherwise data augmentation can even degrade generalization. From the causal perspective (Figure \ref{fig:causal_diagram_rmnist}), irrelevant augmentation also changes the semantic variable $Z$. 

\vspace{-0.3em}
\subsection{Data augmentation and data heterogeneity}\label{sec:grad_norm}
\vspace{-0.3em}

As discussed in \S~\ref{sec:noniid_fl}, federated algorithms may not converge to the optimal solution under non-IIDness.
However, one effect we observe is that with appropriate augmentation, federated learning can achieve comparable performance to the 
centralized setting (see Figure \ref{fig:steps}). 
This is possible because, in such cases, data augmentation increases the similarity of local distributions, shrinking the data heterogeneity (non-IIDness) among clients.

\begin{figure}[h]
    \centering
    \includegraphics[width=\textwidth]{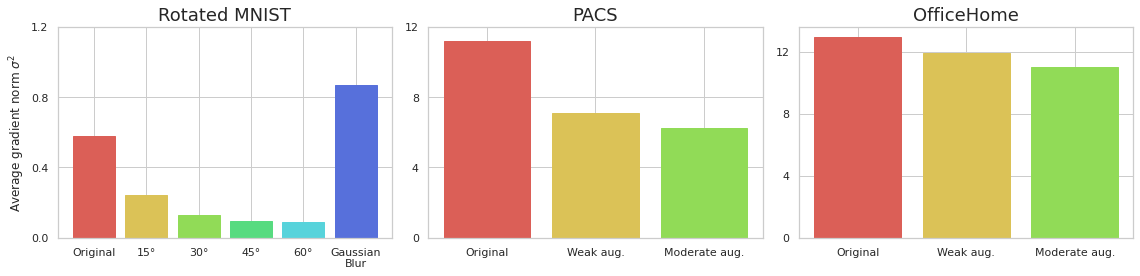}
    \caption{The averaged gradient norm of \textit{FedAvg} across domains under different data augmentation methods, as described in \S~\ref{sec:augmentations}.}
    \label{fig:grad_norm}
\end{figure}

To support this claim, we adopt \eqref{eq:heterogeneity} from \cite{khaled2020tighter} for measuring data heterogeneity.
We use the global model obtained by \textit{FedAvg} in each augmentation scheme as a proxy of the optimal model $\theta^*$.
We estimate the gradient norm of the global model across all in-domain and out-of-domain test sets and take their average.
With increasing levels of proper augmentation, we observe that the averaged gradient norm $\sigma^2$ decreases, indicating that the augmented domains are more similar. However, if the augmentation is not relevant to the environmental variable (see Figure \ref{fig:causal_diagram_rmnist}), such as Gaussian blur for Rotated MNIST, then the augmentation does not help to alleviate the non-IIDness. 

\vspace{-0.3em}
\subsection{Data augmentation for communication efficiency}
\vspace{-0.3em}

The reduction of data heterogeneity has an additional advantage in the federated setting.
By augmenting the clients' data, and thus generating more similar data distributions, local models can be trained in isolation for longer without suffering from weight divergence, as reported in \cite{zhao2018federated}.
To validate this conjecture, we propose the following experiment using the Rotated MNIST dataset and Federated Averaging.
We test the effect of increasing local steps, and decreasing communication rounds, on out-of-domain generalization, with varying degrees of augmentation, as shown in Figure \ref{fig:steps}.
Our results demonstrate that, with the right augmentation, we can obtain comparatively good performance with as few as two rounds of communication, as it is shown in $E=128000$ in Figure \ref{fig:steps}.

\begin{figure}[h]
    \centering
    \includegraphics[width=0.8\textwidth]{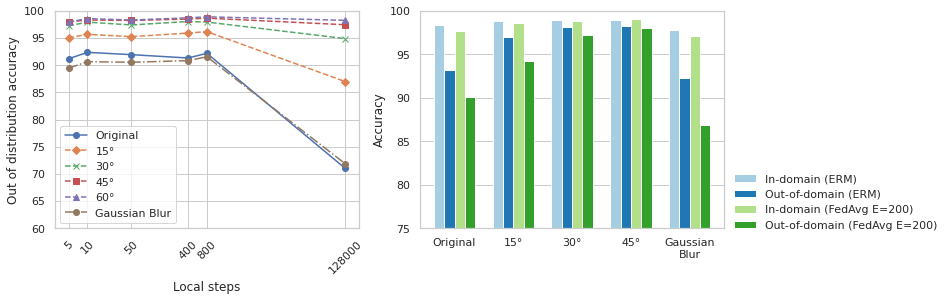}
    \vspace{-0.3em}
    \caption{The effect of data augmentation on out-of-distribution accuracy of \textit{FedAvg} across different steps ({\bf left}) and on the gap between centralized and federated training ({\bf right}).}
    \label{fig:steps}
\end{figure}

In the case where each local gradient step is followed by a communication round, i.e., with $E=1$ local step, federated learning reduces to centralized training. However, under suitable augmentations, we show that the performance gap between centralized and federated learning decreases.
This further motivates the search for effective transformations in more practical settings.

\textbf{In-domain vs.~out-of-domain.} Last but not least, in both centralized and federated cases, we find that in-domain test accuracy is greater than out-of-domain (Figure \ref{fig:steps}), a gap of which is further enlarged in federated training. This emphasizes the challenges of OOD generalization in federated learning. As mentioned before, one potential approach is proper data augmentation, with which we can show that the gap between in-domain and out-of-domain accuracy diminishes. 
%
%

\section{Related work}\label{sec:related_work}


\textbf{Federated Learning:}
In the seminal work \cite{mcmahan2017communication}, despite claiming that \textit{Federated Averaging} (\textit{FedAvg}) is able to cope with non-IID data distributions, extensive empirical evidence suggests a degradation in performance under such settings \cite{hsieh2020non, zhao2018federated}. This is further supported by \cite{li2020convergence} where they demonstrate the convergence of \textit{FedAvg} to be linearly affected by the data heterogeneity.
This is due to the divergence of various client models when trained with data heterogeneity \cite{zhao2018federated}.
Since then, one prolific line of work has been proposing changes to the aggregation scheme of local models \cite{wang2020federated, wang2020tackling, yu2021fed2} or the objective function \cite{li2020federated} to account for these differences.
Other techniques consist of dynamically increasing the communication frequency \cite{hsieh2020non} or introducing a subset of data shared across devices to mitigate model divergence \cite{zhao2018federated}, despite being less viable in practice.
Our formulation of Federated Domain Generalization also relates to the fields of Federated Domain Adaptation \cite{peng2019federated, yao2022federated}  and Federated Transfer Learning \cite{liu2020secure, saha2021federated}, where the concept of different domains is also present.

\textbf{Domain Generalization:} One of the primary assumptions of Domain Generalization is that data is not IID across domains.
However, contrary to Federated Learning, this is typically deemed beneficial, as it can be exploited to design more robust predictors. 
DG techniques typically fall into one of three categories: data manipulation, representation learning, and learning strategy \cite{wang2021generalizing}.
The first aims to learn a general representation across domains by manipulating the inputs either via data augmentation \cite{tobin2017domain, volpi2018generalizing}, or data generation \cite{qiao2020learning, rahman2019multi}.
In representation learning, the goal is to learn a predictor while guaranteeing some stable property across domains, achieved through various techniques, such kernel-based methods \cite{blanchard2017domain}, domain adversarial learning \cite{ganin2016domain}, explicit feature alignment \cite{zhou2020domain}, or invariant predictors \cite{peters2016causal, gong2016domain, arjovsky2019invariant, christiansen2021causal, ilse2021selecting, wang2022out, ilse2021selecting}.
The latter also has formulations connected to FL \cite{francis2021towards}.
The category of learning strategy gathers techniques that employ different strategies to the learning task, such as using meta-learning \cite{finn2017model}, or distributionally robust optimization \cite{sagawa2019distributionally, krueger2021out}, among others.

\textbf{Data Augmentation:} Examples abound on the employment of data augmentation for generalization \cite{shorten2019survey} and self-supervised learning \cite{chen2020simple, he2020momentum, Grill2020bootstrap}. 
Its most elementary use consists of applying handcrafted augmentations, whose success depends on the relationship with the task in hand. For this reason, the choice of augmentation is usually done based on prior knowledge, although some works \cite{ratner2017learning, cubuk2019autoaugment} have attempted to devise a transformation strategy during training.
In more recent work, the goal is to learn such task-related augmentations.
Among different strategies, these may depend on generative models \cite{sixt2016rendergan, antoniou2017data}
or adversarial training \cite{volpi2018generalizing, li2021simple}. In terms of theory, data augmentation can be linked to regularization theory \cite{bishop1995training, todd1994from}, group theory \cite{chen2020group}, as well as to the notion of causality \cite{ilse2021selecting, scholkopf2021towards, wang2022out}.

\textbf{Causal modelling for domain generation}: several approaches attempt to formulate the data generating process of different domains using a causal model.
Examples abound in learning tasks in both causal and anti-causal directions \cite{peters2016causal, heinze2017conditional, rojas2018invariant, gong2016domain, arjovsky2019invariant, scholkopf2012causal}.
Similar to our work, in \cite{heinze2017conditional, mitrovic2020representation, von2021self, ilse2021selecting} the input is causally related to the semantic content and an environmental factor (also denoted as style).
Data augmentation is used in \cite{wang2022out, von2021self, mitrovic2020representation} as a perturbation over the style —also seen as an intervention in \cite{ilse2021selecting}—to promote model invariance.
Further, \cite{besserve2018group} uses group theory to formalize the notion of data augmentation, devising a general principle to learn generative causal models.


\section{Conclusion}

In this paper, we study the out-of-domain generalization of federated learning, and propose to use data augmentation for solving the non-IIDness. This approach is inspired from the connection to domain generalization, from which we can borrow insights for designing FL algorithms as well. Our key observation is that, with proper data augmentation, the client domains are more similar and therefore the FL training suffers less from non-IIDness and OOD generalization can be improved. We demonstrate this point through extensive experiments. Our approach is distinct from the popular line of FL research on better optimization and aggregation, and focuses more on the data side. We hope this data-driven perspective can inspire more exploration into designing proper data augmentation for federated learning and OOD generalization in general. 


\bibliography{refs}

\newpage

\appendix

\section{Connecting Federated Learning and Domain Generalization Algorithms}\label{sec:fl_dg_alg}

In this appendix we provide more details about the FL algorithms that we transfer from the field of domain generalization, including Generalized AFL (Gen-AFL), Variance Minimization (VM) and Federated IRM (Fed-IRM). 
We summarize the connections drawn in this paper between Federated Learning and Domain Generalization algorithms, as well as their objective function formulation in Table \ref{tab:fl_dg_appdx}.

\begin{table}[h]
\caption{FL algorithm vs.~DG algorithm.}
\centering
\begin{tabular}{ccccc}
\toprule 
\bf FL algorithm & \bf DG algorithm & \bf Objective \\
\midrule
FedAvg \cite{mcmahan2017communication} & ERM \cite{vapnik1992principles} & $\sum_i f_i$ \\ 
Agnostic FL \cite{mohri2019agnostic} & GroupDRO \cite{sagawa2019distributionally} & $\max_i f_i$ \\
\it Generalized AFL  & MM-REx \cite{krueger2021out} & $(1 - n \lambda_{\min}) \max_i f_i + \lambda_{\min} \sum_i f_i$ \\
\it Variance Minimization & V-REx \cite{krueger2021out}& $\beta {\rm var}(f) + {\rm mean}(f)$ \\
\it Fed-IRM & IRM \cite{arjovsky2019invariant} & $\sum_i f_i(\sigma(w^\top g)) + \|\nabla_w f_i(\sigma(w^\top g))\|^2$ \\
\bottomrule
\end{tabular}
\label{tab:fl_dg_appdx}
\end{table}

\subsection{Generalized AFL}\label{sec:gen_AFL}

In domain generalization, one extension of GroupDRO is Minimax Risk Extrapolation (MM-REx, \cite{krueger2021out}), where one replaces the simplex $\Delta$ in \eqref{eq:AFL} with $\Delta(\lambda_{\min}) :=\{\lambda\in \Rb^n:\one^\top \lambda = 1, \lambda \geq \lambda_{\min}\one\}$. MM-REx can be treated as a generalization of GroupDRO and ERM: if $\lambda_{\min} = 0$ then we have GroupDRO (AFL), where the highest weight is put on the worst-case client;  if $\lambda_{\min} > 0$, then we want our algorithm to be more uniform (specifically if $\lambda_{\min} = 1/n$, then we have ERM (FedAvg)). The extrapolation occurs when we have $\lambda_{\min} < 0$: we enforce larger weights on the worst-case clients and negative weights on the best clients. In other words, we are extrapolating beyond the convex hull of the client distributions.

In \Cref{alg:AFL} we define Generalized AFL, where $\lambda_{\min} < 1/n$ is a hyperparameter. This algorithm is adapted from MM-REx. The gradient steps can be either SGD or Adam. Note that line 8-10 is based on the following result:
\begin{prop}\label{prop:proof}
Define $\Delta(\lambda_{\min}) :=\{\lambda \in \Rb^n: \one^\top \lambda = 1, \, \lambda \geq \lambda_{\min}\one\}$,  with $\lambda_{\min} < 1/n$. The Euclidean projection of an $n$-dimensional vector $\lambda$ onto $\Delta(\lambda_{\min})$ can be reduced to:
\begin{align}\label{eq:prop_gen_afl}
\proj_{\Delta(\lambda_{\min})}(\lambda) = (1 - n \lambda_{\min}) \proj_{\Delta_0}\left(\frac{\lambda - \lambda_{\min}\one}{1 - n \lambda_{\min}}\right) + \lambda_{\min} \one,
\end{align}
where $\Delta_0$ is the usual $n-1$-dimensional simplex.
\end{prop}

\begin{proof}
The projection on the left hand side of \eqref{eq:prop_gen_afl} can be written as:
\begin{align}\label{eq:argmin}
\proj_{\Delta(\lambda_{\min})}(\lambda) &= \argmin_{\lambda' \in \Delta(\lambda_{\min})} \|\lambda - \lambda'\|.
\end{align}
Note that $\Delta(\lambda_{\min})$ can also be rewritten as:
\begin{align}
\Delta(\lambda_{\min}) &= \left\{\lambda \in \Rb^n: \one^\top \left(\frac{\lambda - \lambda_{\min}\one}{1 - n \lambda_{\min}}\right) = 1, \frac{\lambda - \lambda_{\min}\one}{1 - n \lambda_{\min}}\geq 0 \right\} \nonumber \\
&= \{(1 - n\lambda_{\min})t + \lambda_{\min} \one: t\in \Delta\}.
\end{align}
Therefore, we can rewrite:
\begin{align}
t = \frac{\lambda - \lambda_{\min}\one}{1 - n \lambda_{\min}}, \, t' = \frac{\lambda' - \lambda_{\min}\one}{1 - n \lambda_{\min}}.
\end{align}
With this transformation we can obtain \eqref{eq:prop_gen_afl}. 
\end{proof}

\begin{algorithm}
\caption{Generalized agnostic federated learning (Gen-AFL)}
\label{alg:AFL}
{\bfseries Input: } global epoch $T$, client number $n$, loss function $f_i$, number of samples $n_i$ for client $i$, initial global model $\theta_0$, local step number $E$, learning rate $\eta_\theta$, $\eta_\lambda$, extrapolation parameter $\lambda_{\min} < 1/n$ \\
Let $\lambda_i = \frac{n_i}{N}$ {\bf for} $i$ in $0, 1, \dots, n-1$ \\
\For{$t$ in $0, 1 \dots T - 1$}
{randomly select $\Sc_t \subseteq [n]$ \\
$\x{i}_{t} = \theta_t$ for $i\in \Sc_t$, $N = \sum_{i\in \Sc_t} n_i$ \\
\For(\tcp*[h]{in parallel}){$i$ in $\Sc_t$}
{starting from $\x{i}_{t}$, take $E$ gradient steps on $f_i$ to find $\x{i}_{t+1}$, with learning rate $\eta_\theta$}
$\theta_{t+1} = \sum_{i\in \Sc_t} \lambda_i \x{i}_{t+1}$ \\
$\lambda' = (\lambda + \eta_\lambda f - \lambda_{\min}\one)/(1 - n \lambda_{\min})$ \\
$\lambda = (1 - n \lambda_{\min}){\rm proj}_{\Delta}(\lambda') + \lambda_{\min} \one$ 
}
{\bfseries Output:} global model $\theta_T$
\end{algorithm}

\subsection{Variance Minimization}\label{app:v-rex}

Another extrapolation algorithm proposed in \cite{krueger2021out} is called V-REx, which aims to minimize a linear combination of the variance and the mean:
\begin{align}\label{eq:v-rex}
\min_\theta\  {\rm mean}(F(\theta)) + \beta {\rm var}(F(\theta))  = \frac{1}{n}\sum_i f_i(\theta) +  \frac{\beta}{n} \sum_{i} \left(f_i(\theta) - \frac{1}{n} \sum_j f_j(\theta)\right)^2,
\end{align}
with $\beta > 0$. This objective, in addition to minimizing the total loss, reduces the performance difference among clients, as measured by the variance. Let us derive a federated learning algorithm for this. 

By taking the gradient of \eqref{eq:v-rex} and multiplying by the step size $\eta$, we have:
\begin{align}
\frac{1}{n}\sum_i \eta \nabla f_i(\theta) + 2\frac{\beta}{n} \sum_{i} (f_i(\theta) - \frac{1}{n} \sum_j f_j(\theta))(\eta \nabla f_i(\theta) - \frac{1}{n} \sum_j \eta \nabla f_j(\theta)).
\end{align}
This equation immediately leads to an FL algorithm, by replacing the gradient $\eta \nabla f_i(\theta)$ with the pseudo-gradient (i.e., the opposite of the local update), which we will denote by $\Delta^{(t)}_{i}$:
\begin{align}\label{eq:V-REx}
\frac{1}{n}\sum_i \Delta^{(t)}_{i} + 2\frac{\beta}{n} \sum_{i} (f_i(\theta) - \frac{1}{n} \sum_j f_j(\theta))(\Delta^{(t)}_{i} - \frac{1}{n} \sum_j \Delta^{(t)}_{j}).
\end{align}
The corresponding algorithm is Alg.~\ref{alg:Var_Min}. The gradient steps can be either SGD or Adam. Note that this is a new aggregation rule: instead of simply averaging the local models, it has an additional second term, which relates to the variance of clients. If all clients are identical, this term would vanish. 


\begin{algorithm}
\caption{Variance Minimization (VM)}
\label{alg:Var_Min}
{\bfseries Input: } global epoch $T$, client number $n$, loss function $f_i$, number of samples $n_i$ for client $i$, initial global model $\theta_0$, local step number $E$, learning rate $\eta$, optimizer \\
Let $\lambda_i = \frac{n_i}{N}$ {\bf for} $i$ in $0, 1, \dots, n-1$ \\
\For{$t$ in $0, 1 \dots T - 1$}
{randomly select $\Sc_t \subseteq [n]$ \\
$\x{i}_{t} = \theta_t$ for $i\in \Sc_t$, $N = \sum_{i\in \Sc_t} n_i$ \\
\For(\tcp*[h]{in parallel}){$i$ in $\Sc_t$}
{starting from $\x{i}_{t}$, take $E$ gradient steps on $f_i$ to find $\x{i}_{t+1}$, with learning rate $\eta$\\
compute $\Delta^{(t)}_i =  \x{i}_{t} - \x{i}_{t+1}$}
compute $\Delta_{t} = \frac{1}{n}\sum_i \lambda_i \Delta^{(t)}_{i} + 2\frac{\beta}{n} \sum_{i} \lambda_i (f_i(\theta) - \frac{1}{n} \sum_j f_j(\theta))(\Delta^{(t)}_{i} - \frac{1}{n} \sum_j \Delta^{(t)}_{j})$ \\
$\theta_{t+1} = \theta_t - \Delta_t$ \\
}
{\bfseries Output:} global model $\theta_T$
\end{algorithm}

\subsection{Federated IRM}\label{sec:IRM}

For OOD generalization, there is an influential algorithm based on causal relation, named Invariant Risk Minimization (IRM) \cite{arjovsky2019invariant}. The motivation of IRM is to combat overfitting by avoiding spurious correlations, more commonly obtained with ERM.
It aims to find an invariant classifier that remains optimal across different environments. In deep learning, the predictor is composed of a feature encoder $g$ and a classifier $h$ on top. The formulation of IRM is:
\begin{align}\label{eq:IRM_original}
\min_{h, g} \sum_i f_i(h\circ g), \textrm{ such that } h\in \argmin_{h'} f_i(h'\circ g) \textrm{ for all }i\in [n].
\end{align}

In \cite{arjovsky2019invariant}, the authors propose to treat $h$ as a fixed scalar classifier, and thus the optimization problem becomes:
\begin{align}\label{eq:IRM_v1}
\min_{\theta} \sum_i f_i(\theta) + \beta \|\nabla_{w = 1.0} f_i(w\cdot \theta)\|^2.
\end{align}

Following a similar procedure of FedAvg, we propose the federated version of IRM in Algorithm \ref{alg:fed-irm}.
It is important to mention that other federated versions of IRM have been proposed in \cite{francis2021towards}.
However, these differ from our approach, as they comprise some form of data communication, either through intermediate representations of the feature encoder (as in CausalFed), or through a subset shared with all clients and server (as in CausalFedGSD). For this reason, we did not include these methods in our analysis.
Nevertheless, our Fed-IRM formulation resembles CausalFedGSD, with the distinction that Fed-IRM does not share any data between clients nor the server.

In our formulation of the federated version, when $E=1$ we have the exact agreement to the DG formulation.
However, once local models are trained under IRM with less frequent communication rounds, both in-domain and out-of-domain performance start to degrade, as we further discuss in the results of Rotated MNIST, in section \ref{sec:benchmark}.

\begin{algorithm}[tb]
\caption{Fed-IRM}
\label{alg:fed-irm}
{\bfseries Input:} global epoch $T$, client number $n$, loss function $f_i$ for client $i$, number of samples $n_i$ for client $i$, initial global model $\theta_0$, local step number $E$
\\
\For{$t$ {\bfseries in} $0, 1 \dots T - 1$} 
{randomly select $\Sc_t \subseteq [n]$ \\
$\x{i}_{t} = \theta_{t}$ for $i\in \Sc_t$, $N = \sum_{i\in \Sc_t} n_i$\\
\For(\tcp*[h]{in parallel}){$i$ {\bfseries in} $\Sc_t$}{
    starting from $\x{i}_{t}$, take $E$ gradient steps to minimize $f_i(\theta) + \beta \|\nabla_{w = 1.0} f_i(w\cdot \theta)\|^2$ to find $\x{i}_{t+1}$, with learning rate $\eta$\\
    }
$\theta_{t+1} = \sum_{i\in S_t} \frac{n_i}{N}\x{i}_{t+1}$ \\
}
\textbf{Output:} global model $\theta_T$
\end{algorithm}

\section{Experiments}
\label{sec:appendix-experiments}
We present the experimental details in this section.
Part of the code was based on the DomainBed\footnote{Available at: \texttt{https://github.com/facebookresearch/DomainBed}} test suite, whose license can be found at \texttt{https://github.com/facebookresearch/DomainBed/blob/main/LICENSE}. 
We performed the experiments in a cluster of GPUs, including NVIDIA V100 and P100.

\subsection{Benchmark experiments}\label{sec:benchmark}
In this subsection, we report details over the experiments concerning federated and domain generalization algorithms, as reported in Table \ref{tab:fl_dg_appdx}, also including FedProx, which does not possess a centralized counterpart.
In these experiments, we trained all models for 16000 gradient steps, which assured convergence in all feasible cases.
In federated learning, these gradient steps are further divided into 80 communication rounds, totalizing 200 local steps per round.
Apart from FedProx, which implements the proximal operator with SGD, we train all other algorithms with Adam using the default parameters as reported in the \href{https://pytorch.org/docs/stable/generated/torch.optim.Adam.html#torch.optim.Adam
}{PyTorch library} ($\beta_1 = 0.9$, $\beta_2 = 0.999$), except for the learning rate, which was adapted for each task (see \S~\ref{sec:exp}). 

\begin{table}[h]
\centering
\caption{In domain (ID) and out of domain (OOD) accuracy for Rotated MNIST.}
\label{tab:rminst-da}
\resizebox{\textwidth}{!}{%
\begin{tabular}{llcccccccccccc}
\hline
Training & Augmentation & \multicolumn{2}{c}{ERM/FedAvg} & \multicolumn{2}{c}{GroupDRO/AFL} & \multicolumn{2}{c}{MM-REx/Gen-AFL} & \multicolumn{2}{c}{V-REx/VM} & \multicolumn{2}{c}{IRM/Fed-IRM} & \multicolumn{2}{c}{- /FedProx} \\ \hline
 &  & ID & OOD & ID & OOD & ID & OOD & ID & OOD & ID & OOD & ID & OOD \\ \cline{3-14} 
 & - & 98.4 & \cellcolor[HTML]{EFEFEF}93.2 & 98.2 & \cellcolor[HTML]{EFEFEF}93.7 & 97.2 & \cellcolor[HTML]{EFEFEF}90.0 & 98.3 & \cellcolor[HTML]{EFEFEF}93.1 & 98.2 & \cellcolor[HTML]{EFEFEF}95.6 & - & \cellcolor[HTML]{EFEFEF}- \\
 & 15$^\circ$ rotation & 98.8 & \cellcolor[HTML]{EFEFEF}97.0 & 98.5 & \cellcolor[HTML]{EFEFEF}96.9 & 97.5 & \cellcolor[HTML]{EFEFEF}94.2 & 98.6 & \cellcolor[HTML]{EFEFEF}97.0 & 98.0 & \cellcolor[HTML]{EFEFEF}97.6 & - & \cellcolor[HTML]{EFEFEF}- \\
Centralized & 30$^\circ$ rotation & 98.9 & \cellcolor[HTML]{EFEFEF}98.1 & 98.7 & \cellcolor[HTML]{EFEFEF}97.9 & 96.8 & \cellcolor[HTML]{EFEFEF}95.1 & 98.7 & \cellcolor[HTML]{EFEFEF}98.1 & 98.4 & \cellcolor[HTML]{EFEFEF}98.3 & - & \cellcolor[HTML]{EFEFEF}- \\
 & 45$^\circ$ rotation & 98.9 & \cellcolor[HTML]{EFEFEF}98.3 & 98.5 & \cellcolor[HTML]{EFEFEF}98.0 & 97.1 & \cellcolor[HTML]{EFEFEF}96.6 & 98.6 & \cellcolor[HTML]{EFEFEF}\textbf{98.5} & 98.3 & \cellcolor[HTML]{EFEFEF}98.3 & - & \cellcolor[HTML]{EFEFEF}- \\ \cline{2-14} 
 & Gaussian blur & 97.8 & \cellcolor[HTML]{EFEFEF}92.3 & 98.1 & \cellcolor[HTML]{EFEFEF}92.4 & 96.8 & \cellcolor[HTML]{EFEFEF}89.8 & 98.0 & \cellcolor[HTML]{EFEFEF}92.7 & 97.8 & \cellcolor[HTML]{EFEFEF}95.7 & - & \cellcolor[HTML]{EFEFEF}- \\ \hline
 & - & 97.7 & \cellcolor[HTML]{EFEFEF}90.1 & 97.6 & \cellcolor[HTML]{EFEFEF}92.0 & 97.9 & \cellcolor[HTML]{EFEFEF}92.0 & 97.6 & \cellcolor[HTML]{EFEFEF}89.9 & - & \cellcolor[HTML]{EFEFEF}- & 91.8 & \cellcolor[HTML]{EFEFEF}71.0 \\
 & 15$^\circ$ rotation & 98.6 & \cellcolor[HTML]{EFEFEF}94.2 & 98.1 & \cellcolor[HTML]{EFEFEF}95.3 & 98.3 & \cellcolor[HTML]{EFEFEF}95.1 & 98.3 & \cellcolor[HTML]{EFEFEF}94.0 & - & \cellcolor[HTML]{EFEFEF}- & 92.0 & \cellcolor[HTML]{EFEFEF}76.0 \\
Federated & 30$^\circ$ rotation & 98.8 & \cellcolor[HTML]{EFEFEF}97.2 & 98.4 & \cellcolor[HTML]{EFEFEF}97.2 & 98.8 & \cellcolor[HTML]{EFEFEF}97.3 & 98.7 & \cellcolor[HTML]{EFEFEF}97.3 & - & \cellcolor[HTML]{EFEFEF}- & 92.0 & \cellcolor[HTML]{EFEFEF}83.3 \\
 & 45$^\circ$ rotation & 99.0 & \cellcolor[HTML]{EFEFEF}98.0 & 98.5 & \cellcolor[HTML]{EFEFEF}97.9 & 98.8 & \cellcolor[HTML]{EFEFEF}97.8 & 98.7 & \cellcolor[HTML]{EFEFEF}\textbf{98.3} & - & \cellcolor[HTML]{EFEFEF}- & 91.7 & \cellcolor[HTML]{EFEFEF}86.9 \\ \cline{2-14} 
 & Gaussian blur & 97.1 & \cellcolor[HTML]{EFEFEF}86.9 & 97.0 & \cellcolor[HTML]{EFEFEF}89.6 & 97.5 & \cellcolor[HTML]{EFEFEF}91.6 & 97.3 & \cellcolor[HTML]{EFEFEF}86.7 & - & \cellcolor[HTML]{EFEFEF}- & 88.7 & \cellcolor[HTML]{EFEFEF}67.4 \\ \hline
\end{tabular}%
}
\end{table}

The hyperparameters of the DG algorithms were chosen according to the DomainBed implementation. For the FL algorithms, hyperparameter tuning is used for Gen-AFL as well as VM.
For Gen-AFL, $\lambda_{\min}=-1$ is used for all algorithms, while for VM, $\beta=10$ is used in Rotated MNIST and $\beta=0.5$ is used for PACS and OfficeHome.

The results for Fed-IRM are not reported in Table \ref{tab:rminst-da} due to the challenges in tuning the model locally.
One important distinction from the centralized version is that, in local training, the penalty term comprises only the gradient on the local domain/client.
What we observe is that, even in early communication rounds, the local update does not represent the gradient of the IRM risk well enough.
This is likely due to the sensitivity of the penalty term, since this discrepancy does not significantly affect \textit{FedAvg}/ERM.
Additionally, as we increase the number of local steps, this behavior became more pronounced.
In Rotated MNIST with 45$^\circ$ augmentation, the out-of-domain accuracy of one-step Fed-IRM was reported 91.1\%, while using the standard local step ($E=200$) has yielded 11\%.
Because of this inconsistency, IRM was taken out of the subsequent analyses and experiments.

\begin{table}[h]
\centering
\caption{In domain (ID) and out of domain (OOD) accuracy for PACS.}
\label{tab:pacs-da}
\resizebox{\textwidth}{!}{%
\begin{tabular}{llcccccccccc}
\hline
Training    & Augmentation & \multicolumn{2}{l}{ERM/FedAvg}               & \multicolumn{2}{l}{GroupDRO/AFL}             & \multicolumn{2}{l}{MM-Rex/Gen-AFL}   & \multicolumn{2}{l}{V-Rex/VM}        & \multicolumn{2}{l}{-/FedProx}       \\ \hline
            &              & ID   & OOD                                   & ID   & OOD                                   & ID   & OOD                          & ID   & OOD                          & ID   & OOD                          \\ \cline{3-12} 
            & -            & 92.9 & \cellcolor[HTML]{EFEFEF}50.9          & 92.9 & \cellcolor[HTML]{EFEFEF}51.0          & 91.1 & \cellcolor[HTML]{EFEFEF}62.3 & 92.4 & \cellcolor[HTML]{EFEFEF}52.9 & -    & \cellcolor[HTML]{EFEFEF}-    \\
Centralized & Weak         & 92.3 & \cellcolor[HTML]{EFEFEF}54.2          & 90.8 & \cellcolor[HTML]{EFEFEF}57.8          & 92.6 & \cellcolor[HTML]{EFEFEF}67.3 & 93.8 & \cellcolor[HTML]{EFEFEF}62.8 & -    & \cellcolor[HTML]{EFEFEF}-    \\
            & Moderate     & 94.0 & \cellcolor[HTML]{EFEFEF}71.4          & 93.9 & \cellcolor[HTML]{EFEFEF}\textbf{74.6} & 90.1 & \cellcolor[HTML]{EFEFEF}69.7 & 92.9 & \cellcolor[HTML]{EFEFEF}72.9 & -    & \cellcolor[HTML]{EFEFEF}-    \\ \hline
            & -            & 90.7 & \cellcolor[HTML]{EFEFEF}71.9          & 89.4 & \cellcolor[HTML]{EFEFEF}73.4          & 90.6 & \cellcolor[HTML]{EFEFEF}71.5 & 89.3 & \cellcolor[HTML]{EFEFEF}69.7 & 76.7 & \cellcolor[HTML]{EFEFEF}48.6 \\
Federated   & Weak         & 92.4 & \cellcolor[HTML]{EFEFEF}73.3          & 92.8 & \cellcolor[HTML]{EFEFEF}72.4          & 92.4 & \cellcolor[HTML]{EFEFEF}70.5 & 91.4 & \cellcolor[HTML]{EFEFEF}70.0 & 83.4 & \cellcolor[HTML]{EFEFEF}51.2 \\
            & Moderate     & 93.5 & \cellcolor[HTML]{EFEFEF}\textbf{80.4} & 94.1 & \cellcolor[HTML]{EFEFEF}79.8          & 93.0 & \cellcolor[HTML]{EFEFEF}79.5 & 91.7 & \cellcolor[HTML]{EFEFEF}78.6 & 77.3 & \cellcolor[HTML]{EFEFEF}54.1 \\ \hline
\end{tabular}
}
\end{table}

Generally, more frequent communication is beneficial in federated learning, whereas in the limit of local step $E=1$ we have the equivalent to centralized training.
Surprisingly, the results in PACS indicate that the federated regime, despite having slightly lower in-domain performance, significantly outperformed centralized training in terms of out-of-domain accuracy.
We conjecture that, in this case, alternating between local training and recurrent model aggregation may act as a regularizer (similar to early stopping and dropout) that is favorable for the test domain at hand.
This interesting phenomenon may require further investigation. 

\begin{table}[h]
\centering
\caption{In domain (ID) and out of domain (OOD) accuracy for OfficeHome.}
\label{tab:office-home-da}
\resizebox{\textwidth}{!}{%
\begin{tabular}{llcccccccccc}
\hline
Training    & Augmentation & \multicolumn{2}{l}{ERM/FedAvg}      & \multicolumn{2}{l}{GroupDRO/AFL}             & \multicolumn{2}{l}{MM-Rex/Gen-AFL}   & \multicolumn{2}{l}{V-Rex/VM}                 & \multicolumn{2}{l}{-/FedProx}       \\ \hline
            &              & ID   & OOD                          & ID   & OOD                                   & ID   & OOD                          & ID   & OOD                                   & ID   & OOD                          \\ \cline{3-12} 
            & -            & 74.1 & \cellcolor[HTML]{EFEFEF}38.0 & 73.4 & \cellcolor[HTML]{EFEFEF}37.9          & 48.4 & \cellcolor[HTML]{EFEFEF}24.2 & 73.1 & \cellcolor[HTML]{EFEFEF}38.8          & -    & \cellcolor[HTML]{EFEFEF}-    \\
Centralized & Weak         & 70.3 & \cellcolor[HTML]{EFEFEF}36.8 & 73.2 & \cellcolor[HTML]{EFEFEF}38.5          & 41.7 & \cellcolor[HTML]{EFEFEF}20.6 & 74.6 & \cellcolor[HTML]{EFEFEF}40.7          & -    & \cellcolor[HTML]{EFEFEF}-    \\
            & Moderate     & 74.8 & \cellcolor[HTML]{EFEFEF}41.1 & 74.1 & \cellcolor[HTML]{EFEFEF}40.5          & 41.5 & \cellcolor[HTML]{EFEFEF}23.0 & 76.0 & \cellcolor[HTML]{EFEFEF}\textbf{41.5} & -    & \cellcolor[HTML]{EFEFEF}-    \\ \hline
            & -            & 70.5 & \cellcolor[HTML]{EFEFEF}35.1 & 70.3 & \cellcolor[HTML]{EFEFEF}35.6          & 70.3 & \cellcolor[HTML]{EFEFEF}35.9 & 66.9 & \cellcolor[HTML]{EFEFEF}33.5          & 41.5 & \cellcolor[HTML]{EFEFEF}17.1 \\
Federated   & Weak         & 69.2 & \cellcolor[HTML]{EFEFEF}37.4 & 69.4 & \cellcolor[HTML]{EFEFEF}36.7          & 69.4 & \cellcolor[HTML]{EFEFEF}36.8 & 69.1 & \cellcolor[HTML]{EFEFEF}35.7          & 39.6 & \cellcolor[HTML]{EFEFEF}16.7 \\
            & Moderate     & 71.7 & \cellcolor[HTML]{EFEFEF}40.0 & 71.6 & \cellcolor[HTML]{EFEFEF}\textbf{40.4} & 72.0 & \cellcolor[HTML]{EFEFEF}38.6 & 70.6 & \cellcolor[HTML]{EFEFEF}39.2          & 39.8 & \cellcolor[HTML]{EFEFEF}18.4 \\ \hline
\end{tabular}
}
\end{table}

As for OfficeHome, the task is much harder than in PACS.
Despite having slightly more samples, in OfficeHome, these are distributed along a much larger set of labels than in PACS.
In federated learning, this challenge becomes even harder.
After the dataset split, the resulting clients possess a small fraction of examples for each class, which results in poor generalization for both in and out-of-domain.

The results of OfficeHome and PACS differ from the ones found on \cite{gulrajani2020search} for multiple reasons.
The first distinction is model capacity, as in \cite{gulrajani2020search}, experiments were conducted with ResNet50 networks instead of ResNet18.
Furthermore, \cite{gulrajani2020search} further performs an extensive hyperparameter search as well as model selection based on different criteria.

\subsection{Data augmentation and local steps}
In this experiment, the objective is to validate whether data augmentation could aid the homogenization of clients, thus enabling the federated training with sparser communication rounds.
For this, we trained \textit{FedAvg} with varying local steps, communication rounds and data transformations over Rotated MNIST, whose results are reported in Table \ref{tab:local-steps}.

\begin{table*}[h]
\centering
\caption{The effect of different local steps on in and out of domain accuracy of \textit{FedAvg} across different augmentation schemes over Rotated MNIST.}
\label{tab:local-steps}
\resizebox{\textwidth}{!}{%
\begin{tabular}{cccccccccccccc} 
\cline{3-14}
\multicolumn{2}{l}{}                                                       & \multicolumn{12}{c}{Augmentation}                                                                                                                                                                                                                                                                          \\
\multicolumn{1}{l}{}                     & \multicolumn{1}{l}{}            & \multicolumn{2}{c}{-}                           & \multicolumn{2}{c}{15$^\circ\,$rotation}        & \multicolumn{2}{c}{30$^\circ\,$rotation}        & \multicolumn{2}{c}{45$^\circ\,$rotation}        & \multicolumn{2}{c}{60$^\circ\,$rotation}        & \multicolumn{2}{c}{Gaussian Blur}                \\ 
\hline
\multicolumn{1}{l}{Communication rounds} & \multicolumn{1}{l}{Local steps} & ID   & OOD                                      & ID   & OOD                                      & ID   & OOD                                      & ID   & OOD                                      & ID   & OOD                                      & ID   & OOD                                       \\ 
\hline
2000                                     & 5                               & 97.6 & {\cellcolor[rgb]{0.937,0.937,0.937}}91.2 & 98.1 & {\cellcolor[rgb]{0.937,0.937,0.937}}95.0 & 98.3 & {\cellcolor[rgb]{0.937,0.937,0.937}}97.2 & 98.4 & {\cellcolor[rgb]{0.937,0.937,0.937}}98.0 & 98.1 & {\cellcolor[rgb]{0.937,0.937,0.937}}97.9 & 97.4 & {\cellcolor[rgb]{0.937,0.937,0.937}}89.4  \\
2000                                     & 10                              & 98.0 & {\cellcolor[rgb]{0.937,0.937,0.937}}92.4 & 98.4 & {\cellcolor[rgb]{0.937,0.937,0.937}}95.7 & 98.6 & {\cellcolor[rgb]{0.937,0.937,0.937}}97.9 & 98.6 & {\cellcolor[rgb]{0.937,0.937,0.937}}98.3 & 98.5 & {\cellcolor[rgb]{0.937,0.937,0.937}}98.6 & 976  & {\cellcolor[rgb]{0.937,0.937,0.937}}90.6  \\
320                                      & 50                              & 97.9 & {\cellcolor[rgb]{0.937,0.937,0.937}}91.9 & 98.4 & {\cellcolor[rgb]{0.937,0.937,0.937}}95.2 & 98.5 & {\cellcolor[rgb]{0.937,0.937,0.937}}97.4 & 98.6 & {\cellcolor[rgb]{0.937,0.937,0.937}}98.2 & 98.6 & {\cellcolor[rgb]{0.937,0.937,0.937}}98.3 & 976  & {\cellcolor[rgb]{0.937,0.937,0.937}}90.5  \\
320                                      & 100                             & 98.1 & {\cellcolor[rgb]{0.937,0.937,0.937}}91.6 & 98.6 & {\cellcolor[rgb]{0.937,0.937,0.937}}95.7 & 98.8 & {\cellcolor[rgb]{0.937,0.937,0.937}}97.9 & 98.9 & {\cellcolor[rgb]{0.937,0.937,0.937}}98.5 & 98.9 & {\cellcolor[rgb]{0.937,0.937,0.937}}98.7 & 97.8 & {\cellcolor[rgb]{0.937,0.937,0.937}}90.7  \\
320                                      & 200                             & 98.1 & {\cellcolor[rgb]{0.937,0.937,0.937}}92.3 & 98.6 & {\cellcolor[rgb]{0.937,0.937,0.937}}95.5 & 98.8 & {\cellcolor[rgb]{0.937,0.937,0.937}}97.9 & 98.9 & {\cellcolor[rgb]{0.937,0.937,0.937}}98.6 & 98.7 & {\cellcolor[rgb]{0.937,0.937,0.937}}98.8 & 97.4 & {\cellcolor[rgb]{0.937,0.937,0.937}}90.8  \\
320                                      & 400                             & 97.7 & {\cellcolor[rgb]{0.937,0.937,0.937}}91.3 & 98.5 & {\cellcolor[rgb]{0.937,0.937,0.937}}95.9 & 98.9 & {\cellcolor[rgb]{0.937,0.937,0.937}}98.0 & 98.9 & {\cellcolor[rgb]{0.937,0.937,0.937}}98.5 & 99.0 & {\cellcolor[rgb]{0.937,0.937,0.937}}98.7 & 97.5 & {\cellcolor[rgb]{0.937,0.937,0.937}}90.8  \\
320                                      & 800                             & 98.0 & {\cellcolor[rgb]{0.937,0.937,0.937}}92.2 & 98.7 & {\cellcolor[rgb]{0.937,0.937,0.937}}96.2 & 98.8 & {\cellcolor[rgb]{0.937,0.937,0.937}}97.9 & 99.0 & {\cellcolor[rgb]{0.937,0.937,0.937}}98.6 & 98.9 & {\cellcolor[rgb]{0.937,0.937,0.937}}98.9 & 97.7 & {\cellcolor[rgb]{0.937,0.937,0.937}}91.5  \\
2                                        & 128000                          & 94.3 & {\cellcolor[rgb]{0.937,0.937,0.937}}71.0 & 97.1 & {\cellcolor[rgb]{0.937,0.937,0.937}}87.0 & 98.4 & {\cellcolor[rgb]{0.937,0.937,0.937}}94.9 & 98.8 & {\cellcolor[rgb]{0.937,0.937,0.937}}97.4 & 98.9 & {\cellcolor[rgb]{0.937,0.937,0.937}}98.2 & 93.5 & {\cellcolor[rgb]{0.937,0.937,0.937}}71.9  \\
\hline
\end{tabular}
}
\end{table*}

\subsection{Data augmentation and data heterogeneity}
In this analysis, we use the models trained in the experiment aforementioned and compute the gradient norm with respect to each domain. As previously mentioned in \S~\ref{sec:grad_norm}, in the training domains, the gradient is obtained over the validation set of each domain, while the full set is used for the test domain. The results for Rotated MNIST, PACS and OfficeHome are reported in Tables \ref{tab:gradnorm-rmnist}, \ref{tab:gradnorm-pacs}, and \ref{tab:gradnorm-office-home}, respectively.

\begin{table*}[!ht]
\centering
\caption{Gradient norms of \textit{FedAvg} over the different train and test clients of Rotated MNIST, with varying degrees of augmentation.}
\label{tab:gradnorm-rmnist}
\begin{tabular}{lcccccccc}
\cmidrule[\heavyrulewidth]{2-9}
 & \multicolumn{5}{c}{Train clients} & Test client & \multicolumn{1}{l}{} & \multicolumn{1}{l}{} \\ \hline
Augmentation & 0° & 15° & 30° & 45° & 60° & 75° & mean & std \\ \hline
- & 0.940 & 0.233 & 0.120 & 0.149 & 0.532 & 1.485 & 0.577 & 0.542 \\
15$^\circ$ rotation & 0.401 & 0.107 & 0.074 & 0.077 & 0.226 & 0.583 & 0.245 & 0.208 \\
30$^\circ$ rotation & 0.174 & 0.087 & 0.068 & 0.069 & 0.132 & 0.248 & 0.130 & 0.071 \\
45$^\circ$ rotation & 0.114 & 0.081 & 0.088 & 0.066 & 0.090 & 0.129 & 0.095 & 0.023 \\
60$^\circ$ rotation & 0.093 & 0.086 & 0.093 & 0.069 & 0.086 & 0.121 & 0.091 & 0.017 \\ \hline
Gaussian blur & 1.389 & 0.286 & 0.131 & 0.200 & 0.831 & 2.379 & 0.869 & 0.881 \\ \hline
\end{tabular}%
\end{table*}
\newpage

\begin{table}[ht]
\centering
\caption{Gradient norms of \textit{FedAvg} over the different train and test clients of PACS, with varying degrees of augmentation.}
\label{tab:gradnorm-pacs}
\begin{tabular}{lcccccc}
\cline{2-7}
 & \multicolumn{3}{c}{Train clients} & Test client & \multicolumn{1}{l}{} & \multicolumn{1}{l}{} \\ \hline
Augmentation & P & A & C & S & mean & std \\ \hline
- & 7.629 & 14.416 & 5.696 & 17.067 & 11.202 & 3.865 \\
Weak & 5.999 & 7.659 & 3.444 & 11.344 & 7.112 & 1.871 \\
Moderate & 5.937 & 5.632 & 4.300 & 9.112 & 6.245 & 0.857 \\ \hline
\end{tabular}%

\vspace{2em}
\caption{Gradient norms of \textit{FedAvg} over the different train and test clients of OfficeHome, with varying degrees of augmentation.}
\label{tab:gradnorm-office-home}
\begin{tabular}{lcccccc}
\cline{2-7}
             & \multicolumn{3}{c}{Train clients} & Test client & \multicolumn{2}{c}{Total}  \\ \hline
Augmentation & Art      & Product   & RealWorld  & Clipart     & mean                 & std                  \\ \hline
-            & 15.162   & 6.232     & 8.465      & 22.011      & 12.967               & 4.083                \\
Weak         & 16.490   & 4.972     & 8.318      & 18.003      & 11.946               & 4.942                \\
Moderate     & 15.073   & 5.051     & 7.363      & 16.482      & 10.992               & 4.381                \\ \hline
\end{tabular}
\end{table}
\end{document}